\documentclass[nohyperref]{article}

\usepackage{microtype}
\usepackage{graphicx}
\usepackage{subfigure}
\usepackage{booktabs} 

\usepackage{hyperref}

\usepackage[accepted]{2022}

\usepackage{amsmath}
\usepackage{amssymb}
\usepackage{mathtools}
\usepackage{amsthm}

\usepackage[capitalize,noabbrev]{cleveref}

\theoremstyle{plain}
\newtheorem{theorem}{Theorem}[section]

\newtheorem{lemma}[theorem]{Lemma}

\theoremstyle{definition}

\theoremstyle{remark}
\newtheorem{remark}[theorem]{Remark}

\usepackage[textsize=tiny]{todonotes}

\def\X{\textbf{X}}
\def\x{\textbf{x}}
\def\M{\textbf{M}}
\def\m{\textbf{m}}
\def\W{\textbf{W}}
\def\w{\textbf{w}}
\def\Z{\textbf{Z}}
\def\z{\textbf{z}}
\def\Y{\textbf{Y}}
\def\y{\textbf{y}}
\def\H{\textbf{H}}
\def\h{\textbf{h}}
\def\B{\textbf{B}}
\def\b{\textbf{b}}
\def\mH{\mathcal{H}}
\def\R{\mathcal{R}}
\def\E{\mathbb{E}}

\def\beq{\begin{equation}}              
\def\eeq{\end{equation}}
\def\beqr{\begin{eqnarray}}             
\def\eeqr{\end{eqnarray}}               
\def\beqrs{\begin{eqnarray*}}           
\def\eeqrs{\end{eqnarray*}}     

\icmltitlerunning{FragmGAN: Generative Adversarial Nets for Fragmentary Data Imputation and Prediction}
\usepackage[T1]{fontenc}
\usepackage[utf8]{inputenc}
\usepackage{authblk}
\begin{document}

\onecolumn
\title{ \bf FragmGAN: Generative Adversarial Nets for Fragmentary Data\\
Imputation and Prediction}

\author[1]{Fang Fang\thanks{Correspondence to: Fang Fang <ffang@sfs.ecnu.edu.cn>.}}
\author[1]{Shenliao Bao}
\affil[1]{KLATASDS-MOE, School of Statistics, East China Normal University, Shanghai, P.R.China.}

\date{}

\renewcommand*{\Affilfont}{ \normalsize\it} 

\maketitle

\begin{abstract}
Modern scientific research and applications very often encounter ``fragmentary data" which brings big challenges to imputation and prediction. By leveraging the structure of response patterns, we propose a unified and flexible framework based on Generative Adversarial Nets (GAN) to deal with fragmentary data imputation and label prediction at the same time. Unlike most of the other generative model based imputation methods that either have no theoretical guarantee or only consider Missing Completed At Random (MCAR), the proposed FragmGAN has theoretical guarantees for imputation with data Missing At Random (MAR) while no hint mechanism is needed. FragmGAN trains a predictor with the generator and discriminator simultaneously. This linkage mechanism shows significant advantages for predictive performances in extensive experiments.

\end{abstract} 

\section{Introduction}
\label{intro}

Modern scientific research and applications very often encounter data from multiple data sources, and for each data source, various variables can be collected for data analysis. Such increasing data sources bring big opportunities for predicting people's behaviors with huge potential social and commercial benefits. However, these different data sources usually can not be available for every sample, which leads to ``fragmentary data" and brings big challenges to data imputation and label prediction. To be more specific, we introduce two motivating examples that represent the most typically practical scenarios for fragmentary data.

{\bf Inernet Loan}: A leading company of wealth management is exploring its internet loan business and trying to predict the applicants' income for risk management purpose. There are five possibly available data sources (Table \ref{table1}). (i) Card: the credit card information; (ii) Shopping: the shopping history at internet; (iii) Mobile: the monthly bill of mobile phone; (iv) Bureau: the credit report from the Central Bank; (v) Fraud: the information from an anti-fraud platform. However, some applicants are not willing to provide their shopping or mobile information, not all the applicants have credit reports, and many of them are never included in the database of the anti-fraud platform. As a result, there are 10 ``response patterns" in the Internet Loan data as shown in Table \ref{table1}, where ``$\surd$" means the data source is available for the applicants with the corresponding response pattern.

\vskip -0.1in

\begin{table}[H]
\caption{The response patterns of the Internet Loan data.}
\label{table1}
\begin{center}
\begin{small}
\begin{tabular}{cccccc}
		\toprule
		Response& \multicolumn{5}{c}{Data source}   \\
\cline{2-6}
		Pattern   & \mbox{Card} & Shopping & Mobile & Bureau & Fraud  \\
		\hline
		1   & $\surd$   & $\surd$   & $\surd$   & $\surd$   & $\surd$    \\
		2   & $\surd$   & $\surd$   & $\surd$   & $\surd$   &      \\
		3   & $\surd$   & $\surd$   & $\surd$   &     &      \\
		4   & $\surd$   & $\surd$   &     & $\surd$   & $\surd$    \\
		5   & $\surd$   & $\surd$   &     & $\surd$   &     \\
		6   & $\surd$   & $\surd$   &     &     &      \\
		7   & $\surd$   &     & $\surd$   &     &      \\
		8   & $\surd$   &     &     & $\surd$   & $\surd$    \\
		9   & $\surd$   &     &     & $\surd$   &      \\
		10  & $\surd$   &     &     &     &      \\
		\bottomrule
		\end{tabular}%
\end{small}
\end{center}
\end{table}

{\bf ADNI}: The Alzheimers Disease Neuroimaging Initiative {\it http://adni.loni.usc.edu} is a widely used data by researchers for the Alzheimers disease which has four data sources. (i) CSF: cerebrospinal fluid; (ii) PET: positron emission tomography; (iii) MRI: magnetic resonance imaging; (iv) Gene: the gene expression.  As show in Table \ref{table2}, it has 8 different response patterns corresponding to different data availability for each data source.
\begin{table}[H]
\caption{The response patterns of the ADNI data.}
\label{table2}
\begin{center}
\begin{small}
\begin{tabular}{ccccc}
		\toprule
		Response& \multicolumn{4}{c}{Data source}   \\
\cline{2-5}
		Pattern   & \hspace{0.15cm} CSF \hspace{0.15cm}& \hspace{0.15cm}PET \hspace{0.15cm}& \hspace{0.15cm}MRI\hspace{0.15cm} & \hspace{0.15cm}Gene\hspace{0.15cm}  \\
		\hline
		1   & $\surd$   & $\surd$   & $\surd$   & $\surd$      \\
		2   & $\surd$   & $\surd$   & $\surd$   &       \\
		3   & $\surd$   & $\surd$   &    &    $\surd$      \\
		4   &    & $\surd$   &   $\surd$  & $\surd$     \\
		5   &    & $\surd$   &     & $\surd$       \\
		6   &   & $\surd$   & $\surd$     &          \\
		7   &    &     &    &     $\surd$     \\
		8   &   &     &  $\surd$    &       \\
		\bottomrule
			\end{tabular}%
\end{small}
\end{center}
\vskip -0.2in
\end{table}
Such kind of fragmentary data, also known as ``block-wise missing data" in the statistics literature, are very common in the area of risk management, marketing research, social sciences, medical studies and so on. Data imputation and label prediction are two main goals for the analysis of such data. But the extremely high missing rate and complicated missing patterns bring big challenges to the achievement of the goals.

Some work has been done to deal with fragmentary data in both areas of statistics and computer sciences in recent years. From the statistics perspective, methods based on model averaging \cite{fang19}, factor models \cite{zhang20}, generalized methods of moments \cite{xue21}, iterative least squares \cite{lin21} and integrative factor regression \cite{Li21} are proposed. These statistical methods provide useful theoretical properties but exhibit notable shortcomings: (i) They depend on certain statistical models, for example, linear regression models. (ii) They are not flexible in handling mixed data types that include continuous and categorical variables. (iii) Only a couple of methods consider imputation and prediction at the same time.

From the computer science perspective, GAIN \cite{Yoon18} first uses a Generative Adversarial Net (GAN) to impute data Missing Completed At Random (MCAR), which means the missingness occurs entirely at random without depending on any of the variables. MisGAN \cite{misgan} trains a mask generator along with the data generator for imputation. GAMIN \cite{gamin} proposes a generative adversarial multiple imputation network for highly missing data. HexaGAN \cite{hexagan} deals with missing data imputation, conditional generation and semi-supervised learning together. GRAPE \cite{grape} proposes a graph-based framework for data imputation and label prediction. MIWAE \cite{miwae} and Not-MIWAE \cite{notmiwae} propose imputation methods based on variational auto-encoding (VAE) framework instead of GAN. However, these generative methods have various drawbacks. For instance, some of them \cite{gamin, grape, miwae, notmiwae} do not have the theoretical guarantee that the imputed data has the same distribution as the original data. Some of them \cite{Yoon18, misgan, hexagan} only have theoretical results for data MCAR, which is highly unlikely in the practice. Most of them either consider data imputation and label prediction separately or only consider data imputation.

In this paper, by leveraging the structure of response patterns, we propose a ``FragmGAN" for fragmentary data imputation and prediction. The main contributions are:

\begin{itemize}
\item FragmGAN is a unified framework based on GAN to deal with fragmentary data imputation and label prediction at the same time. It's flexible in the sense that (i) It's applicable to both continuous and categorical data and label. (ii) Users can adjust the relative importance of the task of imputation to prediction by an ``adjusting factor".
\item FragmGAN has theoretical guarantees for imputation with data Missing At Random (MAR), which is much more general than MCAR and will be defined in Section \ref{theorem}. Also, the theoretical results do not need a hint mechanism that is required by GAIN.
\item Using similar technical skills, we extend the theoretical results of GAIN to MAR.
\item Other than the generator and discriminator, FragmGAN trains a predictor simultaneously. This linkage mechanism shows significant advantages for predictive performances in extensive experiments.
\end{itemize}

\section{Related Work}

There are lots of discriminative and generative imputation methods that will be considered in our experiments, including Expectation Maximization \cite{EM}, matrix completion \cite{matrix}, MICE \cite{mice}, MissForest \cite{missforest} and Auto-Encoder \cite{AE}.

There are several other GAN based imputation methods. CollaGAN \cite{collagan} proposes a collaborative GAN for missing data imputation but it focuses on image data. WGAIN \cite{wgain}, CGAIN \cite{cgain}, PC-GAIN \cite{pcgain} and S-GAIN \cite{sgain} extend GAIN in various ways. IFGAN \cite{ifgan} conducts missing data imputation using a feature-specific GAN and MCFlow \cite{mcflow} proposes a Monte Carlo flow method for data imputation but no theoretical result is provided. When all the variables are assumed to be categorical, theoretical results of GAN based methods are extended to an uncommon concept of Extended Always Missing At Random \cite{rank}.

Although they are not our main interest, we also mention some other VAE based imputation methods including VAEAC \cite{vaeac}, variational inference of deep subspaces \cite{subspace}, iterative imputation using AE dynamics \cite{dynamics}, VAE using pattern-set mixtures \cite{patternset} and VSAE \cite{vsae}. Some of them only focus on image data. A common disadvantage of VAE based methods is the lack of theoretical guarantee for imputation. Some results of empirical comparison of GAN and VAE based methods are presented \cite{compare}.

\section{GAN-Based Fragmentary Data Imputation}

We first formulate the problem and discuss the method and theory of fragmentary data imputation in this section. The problem of label prediction will be addressed in Section \ref{framework}.

Throughout the paper we usually use bold type letters to denote vectors and use the regular letters for scalars. The upper-case letters are used for random variables and the corresponding lower-case letters are their realizations. Abusing notation slightly, we use a generic notation $p(\cdot)$ or $p(\cdot|\cdot)$ to denote the distribution/probability or conditional distribution/probability for various continuous/categorical variables as long as there is no ambiguity.

\subsection{Imputation Method}

Let $\X=(X_1,\cdots,X_d)$ be the $d$-dimensional data vector of interested variables that could take continuous or categorical values. Note that $d$ is the number of variables but not the number of data sources since each data source may have multiple variables.

Define the mask vector $\M=(M_1,\cdots,M_d)\in\{0,1\}^d$ such that $M_i=1$ means $X_i$ is observed and $M_i=0$ means $X_i$ is missing, $i=1,\cdots,d$. So what we actually observe is $$\tilde\X=\M\odot\X=(M_1X_1,\cdots,M_dX_d),$$ where $\odot$ denotes element-wise multiplication.

Assume overall there are $K$ possible response patterns in the data and define $\W=(W_1,\cdots,W_K)$ as the pattern indicator, where $W_k=1$ if the sample belongs to the $k$th response pattern and $W_k=0$ otherwise, $k=1,\cdots,K$. Note that $\sum_{k=1}^KW_k=1$. In the fragmentary data setting, $\M$ can actually only take $K$ (rather than $2^d$) different values and there is a one-to-one mapping between $\M$ and $\W$. In the two motivating examples, $K=10$ and 8 respectively.

{\bf Generator}

Let $\Z=(Z_1,\cdots,Z_d)$ be a $d$-dimensional noise vector that is independent of all other variables. It is typically taken as Gaussian white noise. We then feed $\tilde\X=\M\odot\X$, $\Z$ and $\W$ into the generator $G$ and obtain
$$\bar\X=G(\M\odot\X,(1-\M)\odot\Z,\W),$$
where $G$ is a function from $\R^d\times\R^d\times\{0,1\}^K$ to $\R^d$. $\bar\X$ is the generated data vector but we are only interested in the missing variables. So the complete data vector after imputation is
$$\hat\X=\M\odot\tilde\X+(1-\M)\odot\bar\X=\M\odot\X+(1-\M)\odot\bar\X.$$
Our target is to make sure the distribution of $\hat\X$ is the same as the distribution of $\X$, i.e., $p(\hat\X)=p(\X)$. The randomness of $Z$ makes our method a random imputation method rather than fixed imputation. Although we focus on single imputation in the paper, but by modeling the distribution of the data, we are able to make multiple imputation to capture the uncertainty for the imputation value \cite{rubin, mice}.

{\bf Discriminator}

The discriminator $D$ tries to figure out which part of $\hat\X$ is from the generator.  The vanilla GAIN \cite{Yoon18} aims to distinguish each component of $\hat\X$ is real (observed) or fake (imputed). It's a hard task since $d$ usually is a large number. Consequently, a hint mechanism, which reveals all but one of the components of $\M$ to $D$, is required for GAIN to solve the model identifiability problem and make sure the generated distribution is what we want.

In the fragmentary data setting, each sample should exactly belong to one of the $K$ response patterns. By leveraging this informative structure, our discriminator $D$ just needs to figure out which pattern $\hat\X$ belongs to. So $D$ is a function from $\R^d$ to $[0,1]^K$ (instead of $[0,1]^d$ in GAIN) such that
$$\hat\W=D(\hat\X)=(\hat W_1,\cdots,\hat W_K)$$
is the predicted probability vector for $\W$, where $\hat W_k$ is the predicted probability that $\hat\X$ is from the $k$th response pattern and $\sum_{k=1}^K\hat W_k=1$.

We train the discriminator $D$ to {\it maximize} the probability of correctly predicting $\W$. On the other hand, the generator $G$ is trained to {\it minimize} the probability of $D$ correctly predicting $\W$. The objective function is defined to be the negative cross-entropy loss
\beq
V(G,D)=\E_{(\hat\X,\W)}\left[\sum_{k=1}^KW_k\log D_k(\hat\X)\right],\label{objective}
\eeq
where $D_k(\hat\X)$ is just $\hat W_k$. Note that the objective function depends on $G$ through $\hat\X$. Then the minimax problem is given by
\beq \min_G\max_D V(G,D).\label{minmax}\eeq

\vspace{0.2cm}
\begin{remark}
The key difference of our imputation method to GAIN is that we use a different objective function by taking the response patterns into consideration. This adjustment makes sure the model is identifiable even no hint mechanism is used as we show in the next subsection.
\end{remark}

\subsection{Theoretical Results}
\label{theorem}

Most previous theoretical results for GAN-based imputation methods including GAIN \cite{Yoon18}, MisGAN \cite{misgan} and HexaGAN \cite{hexagan} are established under the MCAR assumption, which means the missingness occurs entirely at random without depending on any of the variables. This is a very restrictive assumption and rarely satisfied in the real world. In contrast, our theoretical results will be established under the MAR assumption.

Assume $\X$ can be decomposed into $(\X^o,\X^m)$, where $\X^o$ is an always observed subvector of $\X$, and $\X^m$ could be missing. The missing mechanism is characterized \cite{littleandrubin} into three types:
\begin{itemize}
\item Missing Completed At Random (MCAR): $\M$ is independent of $\X$.
\item Missing At Random (MAR): $p(\M|\X)=p(\M|\X^o)$, or equivalently, $\M$ is conditionally independent of $\X^m$ given $\X^o$.
\item Missing Not At Random (MNAR): $p(\M|\X)$ depends on $\X^m$.
\end{itemize}

\begin{remark}
For a random vector $\X$, it could be ambiguous for the definition of MAR. Another way to define MAR is $p(\M|\X)=p(\M|\{X_i \mbox{ s.t. } M_i=1,i=1,\cdots,d\})$. However, since $\M$ appears in both sides of the equation, there is no way to generate a group of independently and identically distributed samples satisfying this equation, unless there exists an always observed subvector $\X^o$ such that $p(\M|\X)=p(\M|\X^o)$. This is the reason why we use the MAR definition as above.
\end{remark}

The complete data vector $\hat\X$ can be decomposed into $(\hat{\X^o},\hat{\X^m})$ correspondingly. Note that $\hat{\X^o}=\X^o$. So
$$p(\hat\X)=p(\hat{\X^o})p(\hat{\X^m}|\hat{\X^o})=p(\X^o)p(\hat{\X^m}|\X^o).$$
To verify that the solution to the minimax problem (\ref{minmax}) satisfies $p(\hat\X)=p(\X)$, we just need to show $p(\hat{\X^m}|\X^o)=p(\X^m|\X^o)$. First we present a lemma. \\

\begin{lemma}
\label{lemma1}
Let $\hat\x$ is a realization of $\hat\X$ such that $p(\hat\x)>0$. For a fixed generator $G$, the $k$th component of the optimal discriminator $D^*(\hat\x)$ to the minimax problem (\ref{minmax}) is given by
$$D^*_k(\hat\x)=p(\W=\w_k^0|\hat\x)$$
for $k=1,\cdots,K$, where $\w_k^0=(0,\cdots,1,\cdots,0)$ is a $K$-dimensional vector with only the $k$th element being 1, and $\W=\w_k^0$ means that the sample belongs to the $k$th response pattern.
\end{lemma}

\begin{proof}
All proofs are provided in the Appendix \ref{A1}.
\end{proof}

We now rewrite (\ref{objective}) by substituting $D^*$ to obtain the objective function for $G$ to minimize:
\beqrs
C(G)\!\!\!\!&=&\!\!\!\!V(D^*,G)\nonumber\\
\!\!\!\!&=&\!\!\!\!\E_{(\hat\X,\W)}\left[\sum_{k=1}^KW_k\log p(\W=\w_k^0|\hat\X)\right].  \label{CG}
\eeqrs

\begin{theorem}
\label{thm1}
A global minimum for $C(G)$ is achieved if and only if
\beq
p(\hat\x^m|\x^o,\W=\w_k^0)=p(\hat\x^m|\x^o) \label{pp1}
\eeq
for each $k\in\{1,\cdots,K\}$ and $\hat\x=(\x^o,\hat\x^m)$ such that $p(\hat\x)>0$ and $p(\x^o|\W=\w_k^0)>0$.
\end{theorem}

It's worthy to mention that Lemma \ref{lemma1} and Theorem \ref{thm1} do not depend on the MAR assumption and they are generally true even under MNAR.

Theorem \ref{thm1} tells us that the optimal generator will generate data so that the conditional distributions of $\hat{\X^m}$ given $\X^o$ across different response patterns are the same. But it does not guarantee $p(\hat{\X^m}|\X^o)=p(\X^m|\X^o)$ yet.

To further explore, we assume the first response pattern is the case that all the variables are observed, i.e., $M_i=1$ for all $i\in\{1,\cdots,d\}$. Note the first response patterns in the two motivating examples are exactly the case. Then given $\W=\w_1^0$, there is no missing variable and we have $\hat{\X^m}=\X^m$. So following (\ref{pp1}), we have
\beq
p(\hat{\X^m}|\X^o)=p(\hat{\X^m}|\X^o,\W=\w_1^0)=p(\X^m|\X^o,\W=\w_1^0).\label{eqn1}
\eeq

Under the MAR assumption, $\M$ is conditionally independent of $\X^m$ given $\X^o$, and so is $\W$ since there is a one-to-one mapping between $\M$ and $\W$. Therefore
\beq
p(\X^m|\X^o,\W=\w_1^0)=p(\X^m|\X^o). \label{eqn2}
\eeq
Combining (\ref{eqn1}) and (\ref{eqn2}) gives us the final theorem that provides theoretical guarantees for our proposed imputation method.

\vspace{0.2cm}
\begin{theorem}
\label{thm2}
Under the MAR assumption, the density solution to (\ref{pp1}) is unique and satisfies
$$p(\hat\x^m|\x^o)=p(\x^m|\x^o).$$
So the distribution of $\hat\X$ is the same as the distribution of $\X$.
\end{theorem}

Compared to GAIN \cite{Yoon18}, our method do not need a hint mechanism for model identifiability. An intuitive explanation is that we just need to classify each sample into one of the $K$ response patterns. It requires much less model parameters than GAIN, in which $d$ binary classifiers need to be modeled if the hint mechanism is not applied.

Our theoretical results are established under MAR assumption while the vanilla GAIN \cite{Yoon18} assumes MCAR. However, we find that GAIN (with hint) also guarantees that $p(\hat\X)=p(\X)$ under the MAR assumption, which is consistent to a recent theoretical result \cite{rank}. We provide a direct proof of this conclusion of GAIN in the Appendix \ref{A2}.

\section{A Unified Framework for Imputation and Prediction}
\label{framework}
Many previous methods including GAIN \cite{Yoon18} consider label prediction as a post-imputation problem, that is, they first impute the data and then develop a prediction model as if the data were fully observed. The disconnection between imputation and prediction mostly likely damages the accuracy of prediction. In this section we propose a unified framework that considers data imputation and label prediction together. The key idea is to train a predictor $P$ with the generator and discriminator simultaneously.

{\bf Predictor}

Let $\Y$ be the interested $q$-dimensional label that could be continuous or categorical. Unlike the semi-supervised learning, the label $\Y$ is assumed to be available for all the training samples. A predictor $P$ is a function from $\R^d$ to $\R^q$ such that
$\hat\Y=P(\hat\X)$
is a predicted value of $\Y$.

To evaluate the prediction performance of $P$, we define a loss function $L(\Y,P(\hat\X))$ where $L$ is from $\R^d\times\R^d$ to $\R$. The explicit form of $L$ depends on the data type of $\Y$ and is very flexible. For example, if $\Y$ is continuous, we may use $L(\Y,P(\hat\X))=\|\Y-P(\hat\X)\|^2$. If $\Y\in\{0,1\}$ is a binary scalar and the predicted value is the probability of being 1, then we may use $L(\Y,P(\hat\X))=-\Y\log P(\hat\X)-(1-\Y)\log(1-P(\hat\X))$.

To train $G$, $D$ and $P$ together, define the linked objective function as
\beq
U(G,D,P)=\gamma V(G,D)+(1-\gamma)\E_{(\Y,\hat\X)}L(\Y,P(\hat\X)), \label{objective1}
\eeq
where $V(G,D)$ is from (\ref{objective}) and $\gamma\in[0,1]$ is an ``adjusting factor" that controls the relative importance of data imputation to label prediction.

The second part of (\ref{objective1}) does not involve $D$, so the target of $D$ is still to {\it maximize} $V(G,D)$. The first part of (\ref{objective1}) does not involve $P$, so the target of $P$ is to {\it minimize} the predictive loss $\E_{(\Y,\hat\X)}L(\Y,P(\hat\X))$. Both parts of (\ref{objective1}) involves $G$, but fortunately they both require $G$ to {\it minimize}. So the minimax optimization problem is give by
\beq\min_P\min_G\max_D U(G,D,P).\label{minmax1}\eeq

The choice of $\gamma$ is quite flexible. If the user is just interested in data imputation, he can take $\gamma=1$ and $U(G,D,P)$ is reduced to $V(G,D)$. If the user is mainly interested in label prediction, he may use a cross-validation procedure to choose an appropriate $\gamma$ or simply take $\gamma=0.5$ which works quite well as shown in the experiments. Note that $\gamma=0$ is not a good choice since it will lead to overfitting. If the user cares about both imputation and prediction, he may decide $\gamma$ by the relative importance of the two tasks in his mind.

\begin{algorithm}[t]
   \caption{Pseudo Code for FragmGAN}
   \label{alg}
\begin{algorithmic}
   \REPEAT
   \STATE \hspace{-0.22cm}{\bf (1) Discriminator optimization}.
   \STATE Draw $k_D$ samples $\{(\tilde\x(j)),\m(j),\w(j)\}_{j=1}^{k_D}$, draw $k_D$ samples of random noise $\{\z(j)\}_{j=1}^{k_D}$
   \FOR{$j=1$ {\bfseries to} $k_D$}
   \STATE $\bar\x(j)\leftarrow G(\tilde\x(j),(1-\m(j))\odot\z(j),\w(j))$
   \STATE $\hat\x(j)\leftarrow \m(j)\odot\tilde\x(j)+(1-\m(j))\odot\bar\x(j)$
   \STATE Generate hint $\h(j)$
   \ENDFOR
   \STATE Update $D$ using stochastic gradient ascent
   \STATE \hspace{0.5cm}$\nabla_D\sum_{j=1}^{k_D}\sum_{k=1}^Kw_k(j)\log D_k(\hat\x(j),\h(j))$
   \STATE \hspace{-0.22cm}{\bf (2) Generator optimization}.
   \STATE Draw $k_G$ samples $\{(\tilde\x(j)),\m(j),\w(j),\y(j)\}_{j=1}^{k_G}$, draw $k_G$ samples of random noise $\{\z(j)\}_{j=1}^{k_G}$
   \STATE Generate hint $\{\h(j)\}_{j=1}^{k_G}$
   \STATE Update $G$ using SGD ($D$ and $P$ are fixed)
   \STATE \hspace{0.3cm}$\nabla_G\sum_{j=1}^{k_G}\gamma[\sum_{k=1}^Kw_k(j)\log D_k(\hat\x(j),\h(j))+\mathcal{L}_M(\m(j),\tilde\x(j),\bar\x(j))]+(1-\gamma)L(\y(j),P(\hat\x(j)))$
   \STATE \hspace{-0.22cm}{\bf (3) Predictor optimization}.
   \STATE Draw $k_P$ samples $\{(\tilde\x(j)),\m(j),\w(j),\y(j)\}_{j=1}^{k_P}$, draw $k_P$ samples of random noise $\{\z(j)\}_{j=1}^{k_P}$
   \FOR{$j=1$ {\bfseries to} $k_P$}
   \STATE $\bar\x(j)\leftarrow G(\tilde\x(j),(1-\m(j))\odot\z(j),\w(j))$
   \STATE $\hat\x(j)\leftarrow \m(j)\odot\tilde\x(j)+(1-\m(j))\odot\bar\x(j)$
   \ENDFOR
   \STATE Update $P$ using SGD ($G$ is fixed)
   \STATE \hspace{0.5cm}$\nabla_P\sum_{j=1}^{k_P}L(\y(j),P(\hat\x(j)))$
   \UNTIL{training loss has converged}
\end{algorithmic}
\end{algorithm}

The pseudo code to implement (\ref{minmax1}) is given in Algorithm \ref{alg}. Several issues are discussed as follows.

First, although the hint mechanism is not required for our theoretical results, it is still empirically helpful. So we also use the hint mechanism \cite{Yoon18} in implementation. The impact of including the hint mechanism or not will be checked in the experiments.

Second, the generator also generates data even for the observed variables, which can be used to check the generation performance. An extra loss function $\mathcal{L}_M:\{0,1\}^d\times\R^d\times\R^d\rightarrow\R$ defined as $\alpha\sum_{i=1}^dM_iL_M(\tilde X_i,\bar X_i)$ is added to $V(G,D)$ for training $G$, where $L_M:\R\times\R\rightarrow\R$ is a user-specified loss function depending on the variable type of $X_i$. The algorithm result is not sensitive to the choice of hyper-parameter $\alpha$. Actually, as long as $\alpha$ is relatively large ($\alpha=10$ in the experiments), its main effect is to force $\bar X_i=X_i$ for the variable with $M_i=1$.

Third, when $\gamma=1$, Algorithm \ref{alg} actually implements (\ref{minmax}) and the post-imputation prediction.

\begin{table*}[t]
\caption{Imputation performance for UCI datasets in terms of RMSE (Average  $\pm$ Std) of imputation error}
\label{table-imp}
\begin{center}
\begin{small}
\begin{tabular}{cccccc}
		\toprule											
Algorithm	&	Breast	&	Spam	&	Letter	&	Credit	&	News	\\
\toprule	
\multicolumn{6}{c}{MCAR, miss rate=20\%}\\											
\hline											
{\bf FragmGAN}	&	{\bf0.0599 $\pm$ .0021}	&{\bf	0.0537 $\pm$ .0014}	&{\bf	0.1251 $\pm$ .0026}	&{\bf	0.1781 $\pm$ .0057}	&{\bf	0.1484 $\pm$ .0008}	\\
FragmGAN no hint	&	0.0715 $\pm$ .0022	&	0.0545 $\pm$ .0007	&	0.1313 $\pm$ .0078	&	0.1833 $\pm$ .0020	&	0.1595 $\pm$ .0049	\\
\hline	
GAIN	&	0.0658 $\pm$ .0030	&	0.0544 $\pm$ .0005	&	0.1295 $\pm$ .0032	&	0.1814 $\pm$ .0033	&	0.1580 $\pm$ .0063	\\
GAIN no hint	&	0.0736 $\pm$ .0036	&	0.0574 $\pm$ .0004	&	0.1338 $\pm$ .0058	&	0.1899 $\pm$ .0056	&	0.1606 $\pm$ .0051	\\
\hline	
MICE	&	0.0872 $\pm$ .0019	&	0.0715 $\pm$ .0011	&	0.1611 $\pm$ .0045	&	0.1875 $\pm$ .0051	&	0.2152 $\pm$ .0095	\\
MissForest	&	0.0608 $\pm$ .0012	&	0.0594 $\pm$ .0003	&	0.1371 $\pm$ .0012	&	0.2033 $\pm$ .0080	&	0.1932 $\pm$ .0049	\\
Matrix	&	0.1148 $\pm$ .0021	&	0.0562 $\pm$ .0012	&	0.1530 $\pm$ .0035	&	0.2449 $\pm$ .0033	&	0.2291 $\pm$ .0061	\\
AE	&	0.0727 $\pm$ .0011	&	0.0620 $\pm$ .0007	&	0.1361 $\pm$ .0020	&	0.2137 $\pm$ .0026	&	0.1963 $\pm$ .0004	\\
EM	&	0.0754 $\pm$ .0027	&	0.0680 $\pm$ .0002	&	0.1679 $\pm$ .0005	&	0.2312 $\pm$ .0008	&	0.2687 $\pm$ .0003	\\
MisGAN	&	0.0707 $\pm$ .0016	&	0.0582 $\pm$ .0004	&	0.1347 $\pm$ .0020	&	0.1913 $\pm$ .0008	&	0.1746 $\pm$ .0034	\\
\toprule										
\multicolumn{6}{c}{MAR, miss rate=20\%}\\											
\hline											
{\bf FragmGAN}	&	0.0667 $\pm$ .0084	&	{\bf 0.0512 $\pm$ .0007	}&	0.1364 $\pm$ .0072	&	{\bf 0.1844 $\pm$ .0029}	&	{\bf 0.1630 $\pm$ .0043}	\\
FragmGAN no hint	&	0.0730 $\pm$ .0063	&	0.0519 $\pm$ .0007	&	0.1495 $\pm$ .0052	&	0.1966 $\pm$ .0046	&	0.1737 $\pm$ .0060	\\
\hline	
GAIN	&	0.0671 $\pm$ .0092	&	0.0526 $\pm$ .0008	&	0.1457 $\pm$ .0084	&	0.1909 $\pm$ .0040	&	0.1690 $\pm$ .0047	\\
GAIN no hint	&	0.0756 $\pm$ .0099	&	0.0526+¡ª0010	&	0.1505 $\pm$ .0055	&	0.1901 $\pm$ .0040	&	0.1747 $\pm$ .0097	\\
\hline	
MICE	&	0.0931 $\pm$ .0060	&	0.0705 $\pm$ .0008	&	0.1531 $\pm$ .0059	&	0.2487 $\pm$ .0078	&	0.2159 $\pm$ .0096	\\
{\bf MissForest}	&	{\bf 0.0625 $\pm$ .0003}	&	0.0576 $\pm$ .0004	&	{\bf 0.1331 $\pm$ .0059}	&	0.2626 $\pm$ .0025	&	0.1977 $\pm$ .0086	\\
Matrix	&	0.1146 $\pm$ .0031	&	0.0556 $\pm$ .0008	&	0.1488 $\pm$ .0048	&	0.2268 $\pm$ .0056	&	0.2337 $\pm$ .0086	\\
AE	&	0.0843 $\pm$ .0038	&	0.0618 $\pm$ .0017	&	0.1427 $\pm$ .0018	&	0.2110 $\pm$ .0024	&	0.2048 $\pm$ .0008	\\
EM	&	0.0869 $\pm$ .0016	&	0.0678 $\pm$ .0030	&	0.1784 $\pm$ .0050	&	0.2412 $\pm$ .0006	&	0.2659 $\pm$ .0010	\\
MisGAN	&	0.0713 $\pm$ .0014	&	0.0577 $\pm$ .0007	&	0.1428 $\pm$ .0037	&	0.2014 $\pm$ .0046	&	0.1883 $\pm$ .0035	\\
\bottomrule	
	\end{tabular}%
\end{small}
\end{center}
\vskip -0.35in
\end{table*}

\begin{figure*}[t]
\begin{center}
\centerline{\includegraphics[width=13cm]{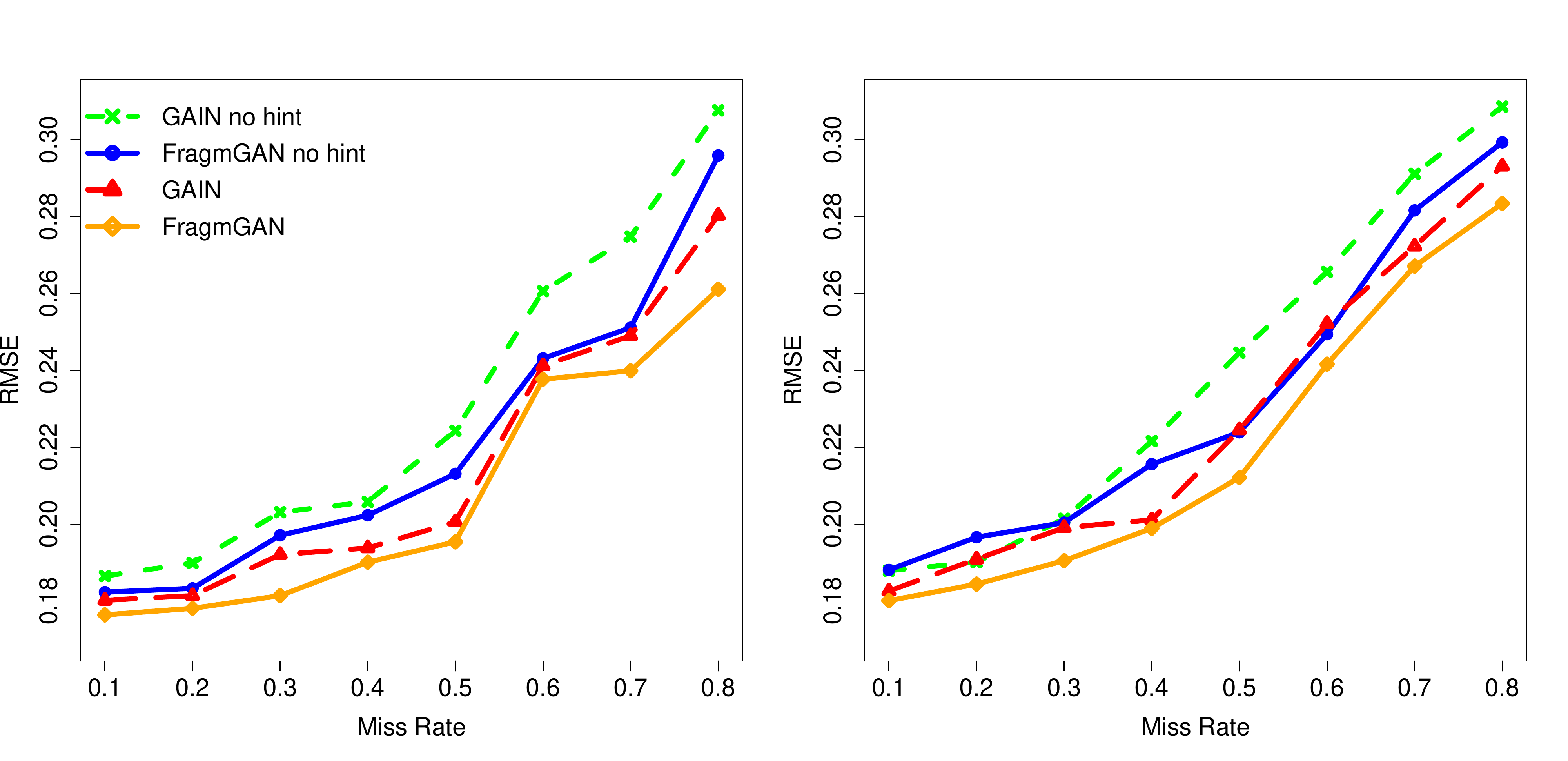}}
\vskip -0.2in
\caption{RMSE of imputation error of the Credit data under different miss rates. Left: MCAR. Right: MAR.}
\label{plot}
\end{center}
\vskip -0.2in
\end{figure*}

\section{Experiments}

In this section we check the imputation and prediction performance of FragmGAN in multiple datasets. First we consider five UCI datasets \cite{UCI} used in GAIN \cite{Yoon18} (Breast, Spam, Letter, Credit and News). Since the original datasets do not have any missing value, we randomly remove part of data by variable groups to make it fragmentary. Unless otherwise stated, the miss rate is 20\%. By designing the removing strategy, we can make it MCAR or MAR. For this group of datasets, we are able to check the performance of data imputation along with label prediction since the true data values are known. Then we consider two datasets Inernet Loan and ADNI for the motivating examples introduced in Section \ref{intro}. The miss rates of them are 46.6\% and 22.3\%, respectively. More details of these two datasets are provided in Appendix \ref{A3} and the data are available in the Supplementary Material. Since the missing values are unknown, we can only check the label prediction performance for these two datasets.

For the purpose of comparison, we consider MICE, MissForest, matrix completion (Matrix), Auto-Encoder (AE), Expectation Maximization (EM) and MisGAN that have been mentioned in Section \ref{intro}. For the prediction task for Inernet Loan and ADNI, we also consider two statistical methods: Model Averaging \cite{fang19}
and FR-FI \cite{zhang20}.

\begin{table*}[t]
\caption{Prediction performance for UCI datasets in terms of AUC (Average  $\pm$ Std)}
\label{table-pre}
\begin{center}
\begin{small}
\begin{tabular}{ccccc}
		\toprule									
Algorithm	&	Breast	&	Spam	&	Credit	&	News	\\
\toprule									
\multicolumn{5}{c}{MCAR, miss rate=20\%}\\									
\hline									
{\bf FragmGAN $\gamma=0.5$}	&{\bf 	0.9932 $\pm$ .0035	}&{\bf 	0.9534 $\pm$ .0029	}&{\bf 	0.7643 $\pm$ .0034	}&	{\bf 0.9709 $\pm$ .0020}	\\
FragmGAN $\gamma=1$	&	0.9920 $\pm$ .0056	&	0.9528 $\pm$ .0030	&	0.7557 $\pm$ .0021	&	0.9620 $\pm$ .0017	\\
\hline
GAIN	&	0.9912 $\pm$ .0055	&	0.9513 $\pm$ .0037	&	0.7521 $\pm$ .0022	&	0.9607 $\pm$ .0026	\\
MICE	&	0.9809 $\pm$ .0045	&	0.9444 $\pm$ .0032	&	0.7492 $\pm$ .0038	&	0.9294 $\pm$ .0021	\\
MissForest	&	0.9892 $\pm$ .0060	&	0.9466 $\pm$ .0069	&	0.7495 $\pm$ .0026	&	0.9409 $\pm$ .0029	\\
Matrix	&	0.9827 $\pm$ .0062	&	0.9021 $\pm$ .0057	&	0.7273 $\pm$ .0074	&	0.8438 $\pm$ .0054	\\
AE	&	0.9850 $\pm$ .0078	&	0.9392 $\pm$ .0053	&	0.7463 $\pm$ .0044	&	0.9211 $\pm$ .0025	\\
EM	&	0.9853 $\pm$ .0029	&	0.9172 $\pm$ .0059	&	0.7418 $\pm$ .0068	&	0.8754 $\pm$ .0028	\\
MisGAN	&	0.9858 $\pm$ .0025	&	0.9485 $\pm$ .0042	&	0.7488 $\pm$ .0017	&	0.9505 $\pm$ .0012	\\
\toprule										
\multicolumn{5}{c}{MAR, miss rate=20\%}\\									
\hline									
{\bf FragmGAN $\gamma=0.5$ }	&{\bf 	0.9936 $\pm$ .0056}	&{\bf  	0.9530 $\pm$ .0034	}&{\bf  	0.7622 $\pm$ .0027}	&{\bf 	0.9696 $\pm$ .0024}	\\
FragmGAN $\gamma=1$	&	0.9928 $\pm$ .0045	&	0.9521 $\pm$ .0029	&	0.7518 $\pm$ .0015	&	0.9598 $\pm$ .0017	\\
\hline	
GAIN	&	0.9914 $\pm$ .0040	&	0.9511 $\pm$ .0032	&	0.7505 $\pm$ .0021	&	0.9592 $\pm$ .0023	\\
MICE	&	0.9878 $\pm$ .0063	&	0.9375 $\pm$ .0036	&	0.7366 $\pm$ .0033	&	0.9325 $\pm$ .0040	\\
MissForest	&	0.9839 $\pm$ .0035	&	0.9519 $\pm$ .0042	&	0.7355 $\pm$ .0026	&	0.9405 $\pm$ .0026	\\
Matrix	&	0.9815 $\pm$ .0083	&	0.9033 $\pm$ .0045	&	0.7342 $\pm$ .0028	&	0.8596 $\pm$ .0036	\\
AE	&	0.9895 $\pm$ .0056	&	0.9347 $\pm$ .0041	&	0.7485 $\pm$ .0056	&	0.9291 $\pm$ .0041	\\
EM	&	0.9892 $\pm$ .0064	&	0.9134 $\pm$ .0036	&	0.7427 $\pm$ .0063	&	0.8828 $\pm$ .0061	\\
MisGAN	&	0.9863 $\pm$ .0023	&	0.9499 $\pm$ .0050	&	0.7483 $\pm$ .0021	&	0.9492 $\pm$ .0020	\\
\bottomrule									
\end{tabular}%
\end{small}
\end{center}
\end{table*}

The hyperparameters of FragmGAN and some implementation details are provided in Appendix \ref{A3}.  More details can be found in the implementation code of FragmGAN that is available in the Supplementary Material.

For each dataset, we randomly split it into a training set (80\%) and a test set (20\%) by response patterns. All the methods are fitted in the training set and then applied to the test set. The imputation and prediction performances are evaluated at the test set. We repeat this experiment 10 times and report the averages and standard deviations of the evaluation criteria (RMSE or AUC). In each table, the best result for each dataset is marked in bold type.

\subsection{Results for the UCI Datasets}

{\bf Imputation Performance}. Table \ref{table-imp} reports the RMSEs of the imputation errors for the UCI datasets. We take $\gamma=1$ for FragmGAN since imputation is the focus here. For both FragmGAN and GAIN, we consider two versions with or without the hint mechanism.

As we can see from Table \ref{table-imp}, FragmGAN outperforms all the other methods in most cases. For the two cases that FragmGAN is not the best (Breast and Letter with MAR, in which MissForest performs the best), it performs the second best. Both FragmGAN and GAIN perform better than their corresponding versions without hint, indicating that the hint mechanism really helps empirically. This is expected since the hint mechanism provides useful information to the discriminator. Note that the results here can not be directly compared to the results in the paper of GAIN \cite{Yoon18} since here we consider fragmentary data with certain response patterns while the missing data in GAIN is generated totally at random.

To check the imputation performance under different miss rates, we take the dataset Credit and generate missing data with miss rate from 10\% to 80\%. Figure \ref{plot} presents the RMSEs of imputation errors under different miss rates. We can see that FragmGAN consistently performs the best. Again, both FragmGAN and GAIN perform better than their corresponding versions without hint. FragmGAN outperforms GAIN in both versions with or without hint.

Overall speaking, FragmGAN performs quite well in data imputation in the sense that it has smaller RMSE of imputation error compared to the competitors.  Specifically it is better than GAIN, indicating that considering the structure of response patterns in the algorithm is really useful.

{\bf Prediction Performance}. Table \ref{table-pre} reports the AUCs for the prediction performance in the datasets Breast, Spam, Credit and News. The dataset Letter is not considered here since it does not have a binary label. We include hint for both FragmGAN and GAIN. The adjusting factor $\gamma$ is taken as 1 or 0.5 for FragmGAN.  Note that when $\gamma=1$, FragmGAN first imputes data and then makes the prediction as if the data were fully observed. When $\gamma=0.5$, the imputation and prediction are considered simultaneously.

As we can see, FragmGAN with $\gamma=0.5$ outperforms the other methods in all the cases. This result shows that the linkage mechanism of training generator and predictor together can improve the prediction performance as we expected. Also note that although FragmGAN with $\gamma=1$ performs worse than FragmGAN with $\gamma=0.5$, it still performs better than all the other methods.

\subsection{Results for the Motivating Examples}

\begin{table*}[t]
\caption{Prediction performance for the two motivation examples (Average  $\pm$ Std)}
\label{table-motivating}
\begin{center}
\begin{small}
\begin{tabular}{cccc}
		\toprule							
Algorithm	&	Internet Loan (RMSE)	&	ADNI (RMSE)	&	ADNI (AUC)	\\
\hline							
{\bf FragmGAN $\gamma_{cv}$	}&	{\bf 0.8865 $\pm$ .0015	}&{\bf 	0.0851 $\pm$ .0019}	&{\bf 	0.7823 $\pm$ .0026 }	\\
FragmGAN $\gamma=1$	&	0.9267 $\pm$ .0036	&	0.0897 $\pm$ .0022	&	0.7701 $\pm$ .0036	\\
FragmGAN $\gamma=0.75$	&	0.9151 $\pm$ .0029	&	0.0883 $\pm$ .0023	&	0.7721 $\pm$ .0028	\\
FragmGAN $\gamma=0.5$	&	0.8928 $\pm$ .0026	&	0.0871 $\pm$ .0020	&	0.7773 $\pm$ .0023	\\
FragmGAN $\gamma=0.25$	&	0.9286 $\pm$ .0044	&	0.0895 $\pm$ .0028	&	0.7719 $\pm$ .0018	\\
\hline	
GAIN	&	0.9246 $\pm$ .0034	&	0.0921 $\pm$ .0014	&	0.7622 $\pm$ .0028	\\
MICE	&	0.9934 $\pm$ .0036	&	0.1034 $\pm$ .0029	&	0.6587 $\pm$ .0029	\\
MissForest	&	0.9982 $\pm$ .0041	&	0.1124 $\pm$ .0019	&	0.7583 $\pm$ .0022	\\
Matrix	&	0.9913 $\pm$ .0039	&	0.1134 $\pm$ .0016	&	0.7343 $\pm$ .0018	\\
AE	&	0.9884 $\pm$ .0035	&	0.0994 $\pm$ .0017	&	0.7400 $\pm$ .0036	\\
EM	&	0.9896 $\pm$ .0034	&	0.1042 $\pm$ .0023	&	0.7020 $\pm$ .0048	\\
MisGAN	&	0.9889 $\pm$ .0021	&	0.0997 $\pm$ .0024	&	0.7384 $\pm$ .0026	\\
Model Averaging	&	0.9831 $\pm$ .0071	&	0.1000 $\pm$ .0095	&	not applicable	\\
FR-FI	&	1.0560 $\pm$ .0112	&	0.1057 $\pm$ .0087	&	not applicable	\\
\bottomrule							
\end{tabular}%
\end{small}
\end{center}
\end{table*}

For the dataset Inernet Loan, the original label is the applicant's income, which is a continuous variable. In the analysis we use $\log(income)$ as the label $Y$. For the dataset ADNI, the original label $Y$ is the score of Mini-Mental State Examination (MMSE) taking value from 0 to 30, in which higher score means better cognitive function. In the real analysis, we consider two labels: (i) The normalized MMSE which can be considered as a continuous variable. (ii) A binary label $Y=1$ if MMSE$\geq$28 and $Y=0$ otherwise.

{\bf Prediction Performance}. Table \ref{table-motivating} reports the RMSEs for the continuous label prediction and AUCs for the binary label prediction. The last two methods (Model Averaging and FR-FI) rely on linear regression models so they are not applicable to the binary label prediction. For the proposed FragmGAN, we take $\gamma=$1, 0.75, 0.5 and 0.25, indicating different relative importance of imputation to prediction. Also, we use a 5-fold cross-validation to select the best (for label prediction) $\gamma$. The CV criterion is defined as the averaged prediction performance in the leave-out samples.

Table \ref{table-motivating} shows that FragmGAN with the $\gamma_{cv}$ selected by cross-validation outperforms all the other methods in all the three cases, indicating that cross-validation is a good way to choose $\gamma$. FragmGAN with $\gamma=0.5$ always performs the second best. Note that $\gamma$ controls the relative importance of data imputation to label prediction. As $\gamma$ decreases from 1 to 0.25, the prediction performances first increase and then decrease. This result confirms two points that we have made: first, the linkage mechanism  of training generator and predictor
together  can improve the prediction performance; second, a small $\gamma$ close to 0 will lead to overfitting and damage the label prediction performance at the test data. Base on the results, we believe $\gamma=0.5$ is a reasonable choice if the users are not willing to apply cross validation to choose the best $\gamma$ due to the computational burden.

\section{Concluding Remarks}

Fragmentary data is becoming more and more popular in many areas and it is not easy to handle. By leveraging the structure in the response patterns, we propose a unified and flexible GAN based framework to deal with data imputation and label prediction simultaneously. An adjusting factor $\gamma$ is used to adjust the relative importance of imputation to prediction. Theoretical guarantees for imputation are provided under the MAR assumption. Extensive experiments confirm the superiority of our proposed FragmGAN. It has wide application prospects especially in personal internet credit investigation, individual patient data (IPD) meta analysis in medical research, and so on.

Based on the theoretical explorations and results of the experiments, we provide several practical suggestions for implementing FragmGAN in the practice:
\begin{itemize}
\item Always use the hint mechanism although it is not required for the theoretical results.
\item Use $\gamma=1$ if there is no label in the analysis or you are only interested in data imputation.
\item If you are interested in label prediction, use cross-validation to select the best $\gamma$ or simply use $\gamma=0.5$. 
\end{itemize}

Possible future work includes: (i) We find that the best performances of data imputation (requires $\gamma=1$) and label prediction (requires a $\gamma$ between 0 and 1) are not achieved at the same time. This is an interesting phenomenon and further investigation may lead to some thought-provoking results. (ii) In this paper we assume that the label is always available in the training data. We my explore FragmGAN under a semi-supervised setting in which some labels are not available. (iii) In this paper we assume that data are MAR. The extension of the results to the more general case of MNAR is a difficult but interesting task.

\vspace{0.4cm}
\bibliography{deeplearning_paper}
\bibliographystyle{2022}

\newpage
\appendix
\onecolumn
\section{Appendix.}

\subsection{Proofs of the Proposed FragmGAN}
\label{A1}

{\bf 1.1. Proof of Lemma \ref{lemma1}}
\begin{proof}
Let $\pi_k=P(\W=\w_k^0)$.
\beqr
V(G,D)&=&\E_{(\hat\X,\W)}\left[\sum_{k=1}^KW_k\log D_k(\hat\X)\right]\nonumber\\
&=&\sum_{k=1}^KP(\W=\w_k^0)\int_{\hat\x}p(\hat\x|\W=\w_k^0)\log D_k(\hat\x)d\hat\x\nonumber\\
&=&\int_{\hat\x}\left[\sum_{k=1}^K\pi_k p(\hat\x|\W=\w_k^0)\log D_k(\hat\x)\right]d\hat\x. \label{eqn3}
\eeqr
Note that $\sum_{k=1}^K\log D_k(\hat\x)=1$. By the fact that $\sum_{k=1}^Kc_k\log x_k$ with $\sum_{k=1}^Kx_k=1$ achieves its maximum when $x_k=\frac{c_k}{\sum_{k=1}^Kc_k}$, (\ref{eqn3}) is maximized (for fixed $G$) when
$$D_k(\hat\x)=\frac{\pi_kp(\hat\x|\W=\w_k^0)}{\sum_{k=1}^K\pi_kp(\hat\x|\W=\w_k^0)}=p(\W=\w_k^0|\hat\x):=D_k^*(\hat\x)$$
for the $\hat\x$ such that $p(\hat\x)>0$.
\end{proof}

{\bf 1.2. Proof of Theorem \ref{thm1}}
\begin{proof}
\beqr
C(G)&=&V(D^*,G)\nonumber\\
&=&\E_{(\hat\X,\W)}\left[\sum_{k=1}^KW_k\log p(\W=\w_k^0|\hat\X)\right]\nonumber\\
&=&\sum_{k=1}^K\pi_k\int_{\hat\x}p(\hat\x|\W=\w_k^0)\log p(\W=\w_k^0|\hat\x)d\hat\x\nonumber\\
&=&\sum_{k=1}^K\pi_k\int_{\hat\x}p(\hat\x|\W=\w_k^0)\log\frac{p(\hat\x|\W=\w_k^0)\pi_k}{p(\hat\x)}d\hat\x\nonumber\\
&=&\sum_{k=1}^K\pi_k\int_{\hat\x}p(\hat\x|\W=\w_k^0)\log\frac{p(\hat\x|\W=\w_k^0)}{p(\hat\x)}d\hat\x+\sum_{k=1}^K\pi_k\int_{\hat\x}p(\hat\x|\W=\w_k^0)\log\pi_k d\hat\x\nonumber\\\nonumber\\
&\propto&\sum_{k=1}^K\pi_k\int_{\hat\x}p(\hat\x^m|\x^o,\W=\w_k^0)p(\x^o|\W=\w_k^0)\log\frac{p(\hat\x^m|\x^o,\W=\w_k^0)p(\x^o|\W=\w_k^0)}{p(\hat\x^m|\x^o)p(\x^o)}d\hat\x\label{eqn4}\\
&=&\sum_{k=1}^K\pi_k\int_{\x^o}p(\x^o|\W=\w_k^0)\left[\int_{\hat\x^m}p(\hat\x_m|\x^o,\W=\w_k^0)\log\frac{p(\hat\x^m|\x^o,\W=\w_k^0)}{p(\hat\x^m|\x^o)}d\hat\x^m\right]d\x^o\nonumber\\
&&\hspace{1cm}+\sum_{k=1}^K\pi_k\int_{\x^o}p(\x^o|\W=\w_k^0)\left[\int_{\hat\x^m}p(\hat\x_m|\x^o,\W=\w_k^0)\log\frac{p(\x^o|\W=\w_k^0)}{p(\x^o)}d\hat\x^m\right]d\x^o,\label{eqn5}
\eeqr
where ``$\propto$" means equation holds by ignoring terms unrelated to $G$, and (\ref{eqn4}) holds since $\int_{\hat\x}p(\hat\x|\W=\w_k^0)d\hat\x=1$ is a constant and $\hat\x=(\x^o,\hat\x^m)$. Note that $\log\frac{p(\x^o|\W=\w_k^0)}{p(\x^0)}$ is unrelated to $\hat\x^m$ and $\int_{\hat\x^m}p(\hat\x_m|\x^o,\W=\w_k^0)d\hat\x^m=1$, so the second term of (\ref{eqn5}) is unrelated to $G$. Following (\ref{eqn5}), we have
$$C(G)\propto\sum_{k=1}^K\pi_k\int_{\x^o}p(\x^o|\W=\w_k^0)\mbox{KL}\Big(p(\hat\x^m|\x^o,\W=\w_k^0)||p(\hat\x^m|\x^o)\Big)d\x^o,$$
where KL$(\cdot|\cdot)$ denotes the KL divergence. Its minimum is achieved when $p(\hat\x^m|\x^o,\W=\w_k^0)=p(\hat\x^m|\x^o)$ for each $k\in\{1,\cdots,K\}$ and (almost) every $\x$ such that $p(\hat\x)>0$ and $p(\x^o|\W=\w_k^0)>0$.
\end{proof}

{\bf 1.3. Proof of Theorem \ref{thm2}}
\begin{proof}
Actually we have proved this theorem in the statements between Theorem \ref{thm1} and Theorem \ref{thm2} in Section \ref{theorem}. \\
\end{proof}

\vspace{0.2cm}
\subsection{Extend Theoretical Results of GAIN \cite{Yoon18} to Missing at Random}
\label{A2}
We first rewrite the formulation of GAIN under MAR with our notation (just a little bit different from the original GAIN paper).

The original data $\X=(\X^o,\X^m)\in\R^d$, dim$(\X^o)=d^o$, dim$(\X^m)=d^m$ and $d^o+d^m=d$. Denote $\M\in\{0,1\}^{d^m}$ as the response indicator for $\X^m$. We assume $\X^m$ is missing at random, i.e., $p(\M|\X)=p(\M|\X^o)$. Let $\Z=(Z_1,\cdots,Z_{d^m})$ be a $d^m$-dimensional noise vector.  Denote
\beqrs
\bar\X^m&=&G(\X^o,\M\odot\X^m,(1-\M)\odot\Z,\M)\in\R^{d^m},\\
\hat{\X^m}&=&\M\odot\X^m+(1-\M)\odot\bar\X^m\in\R^{d^m},
\eeqrs
and $\hat\X=(\X^o,\hat{\X^m})\in\R^d$ is the complete data after imputation. Let $\H\in\R^{d^m}$ be the hint vector. The discriminator $D$ is a function from $\R^d\times\R^{d^m}$ to $[0,1]^{d^m}$ such that
$$\hat\M=D(\hat\X,\H)=(\hat M_1,\cdots,\hat M_{d_m})$$
is the predicted probability vector for $\M$. The minimax problem is:
\beq \min_G\max_D V(G,D)=\min_G\max_D \E_{(\hat\X,\M,\H)}\left[\M^T\log D(\hat\X,\H)+(1-\M)^T\log(1-D(\hat\X,\H)),\right]  \label{objectiveGAIN}  \eeq
where $\log$ is element-wise logarithm and dependence on $G$ is through $\hat\X$.

The proof of {\bf Lemma 1} in GAIN \cite{Yoon18} does not depend on the decomposition of $\X$. So the result still holds: the optimal $D$ for given $G$ is given by
$$D_i^*(\hat\x,\h)=p(M_i=1|\hat\x,\h)$$
for $i\in\{1,\cdots,d^m\}$.

Denote $\mH_t^i=\{\h: p(\h|m_i=t)>0\}$ for $t\in\{0,1\}$ and $i\in\{1,\cdots,d^m\}$. Substituting $D^*=(D_1^*,\cdots,D_{d^m}^*)$ into $V(G,D)$ in (\ref{objectiveGAIN}), we have the objective function for $G$ (to minimize):
\beqr
C(G)&=&\E_{(\hat\X,\M,\H)}\left[\sum_{i:M_i=1}\log p(m_i=1|\hat\X,\H)+\sum_{i:M_i=0}\log p(m_i=0|\hat\X,\H)\right]\nonumber\\
&=&\int_{\hat\x}\int_{\h}\sum_{i=1}^{d^m}\left[p(\hat\x,\h,m_i=1)\log p(m_i=1|\hat\x,\h)+p(\hat\x,\h,m_i=0)\log p(m_i=0|\hat\x,\h)\right]d\h d\hat\x\nonumber\\
&=&\sum_{i=1}^{d_m}\sum_{t\in\{0,1\}}\int_{\mH_t^i}\int_{\hat\x}p(\hat\x,\h,m_i=t)\log p(m_i=t|\hat\x,\h)d\h d\hat\x\nonumber\\
&=&\sum_{i=1}^{d_m}\sum_{t\in\{0,1\}}\int_{\mH_t^i}\int_{\hat\x}p(\hat\x,\h,m_i=t)\log\frac{p(\hat\x,m_i=t|\h)}{p(\hat\x|\h)} d\h d\hat\x\nonumber
\eeqr
\beqr
&=&\sum_{i=1}^{d_m}\sum_{t\in\{0,1\}}\int_{\mH_t^i}\int_{\hat\x}p(\hat\x,\h,m_i=t)\log\frac{p(\hat\x,m_i=t|\h)p(m_i=t|\h)}{p(\hat\x|\h)p(m_i=t|\h)} d\h d\hat\x\nonumber\\
&=&\sum_{i=1}^{d_m}\sum_{t\in\{0,1\}}\int_{\mH_t^i}\int_{\hat\x}p(\hat\x,\h,m_i=t)\log\frac{p(\hat\x|\h,m_i=t)p(m_i=t|\h)}{p(\hat\x|\h)} d\h d\hat\x\nonumber\\
&=&\sum_{i=1}^{d_m}\sum_{t\in\{0,1\}}\int_{\mH_t^i}\int_{\hat\x}p(\hat\x,\h,m_i=t)\log\frac{p(\hat\x|\h,m_i=t)}{p(\hat\x|\h)} d\h d\hat\x\nonumber\\
&&\hspace{3.5cm}+\sum_{i=1}^{d_m}\sum_{t\in\{0,1\}}\int_{\mH_t^i}\int_{\hat\x}p(\hat\x,\h,m_i=t)\log p(m_i=t|\h) d\h d\hat\x. \label{eqn6}
\eeqr
Note that $\log p(m_i=t|\h)$ and $\int_{\hat\x}p(\hat\x,\h,m_i=t)d\hat\x=p(\h,m_i=t)$ are not related to $\hat\x$. So the second term of (\ref{eqn6}) is not related to $G$ and hence
\beqr
&&C(G)\nonumber\\
&\propto&\sum_{i=1}^{d_m}\sum_{t\in\{0,1\}}\int_{\mH_t^i}\int_{\hat\x}p(\hat\x,\h,m_i=t)\log\frac{p(\hat\x|\h,m_i=t)}{p(\hat\x|\h)} d\h d\hat\x\nonumber\\
&=&\sum_{i=1}^{d_m}\sum_{t\in\{0,1\}}\int_{\mH_t^i}p(\h,m_i\!=\!t)\left[\int_{\hat\x}p(\hat\x|\h,m_i=t)\log\frac{p(\hat\x|\h,m_i=t)}{p(\hat\x|\h)} d\hat\x\right]d\h\nonumber\\
&=&\sum_{i=1}^{d_m}\sum_{t\in\{0,1\}}\int_{\mH_t^i}p(\h,m_i\!=\!t)\left[\int_{\hat\x}p(\hat\x^m|\x^o,\h,m_i\!=\!t)p(\x^o|\h,m_i\!=\!t)\log\frac{p(\hat\x^m|\x^o,\h,m_i\!=\!t)p(\x^o|\h,m_i\!=\!t)}{p(\hat\x^m|\x^o,\h)p(\x^o|\h)} d\hat\x\right]d\h\nonumber\\
&=&\sum_{i=1}^{d_m}\sum_{t\in\{0,1\}}\int_{\mH_t^i}p(\h,m_i\!=\!t)\int_{\x^o}p(\x^o|\h,m_i\!=\!t)\left[\int_{\hat\x^m}p(\hat\x^m|\x^o,\h,m_i\!=\!t)\log\frac{p(\hat\x^m|\x^o,\h,m_i\!=\!t)}{p(\hat\x^m|\x^o,\h)} d\hat\x^m\right]d\x^od\h\nonumber\\
&&\vspace{0.2cm}+\sum_{i=1}^{d_m}\sum_{t\in\{0,1\}}\int_{\mH_t^i}p(\h,m_i\!=\!t)\int_{\x^o}p(\x^o|\h,m_i\!=\!t)\left[\int_{\hat\x^m}p(\hat\x^m|\x^o,\h,m_i\!=\!t)\log\frac{p(\x^o|\h,m_i\!=\!t)}{p(\x^o|\h)} d\hat\x^m\right]d\x^od\h\nonumber\\
&\propto&\sum_{i=1}^{d_m}\sum_{t\in\{0,1\}}\int_{\mH_t^i}p(\h,m_i\!=\!t)\int_{\x^o}p(\x^o|\h,m_i\!=\!t)\left[\int_{\hat\x^m}p(\hat\x^m|\x^o,\h,m_i\!=\!t)\log\frac{p(\hat\x^m|\x^o,\h,m_i\!=\!t)}{p(\hat\x^m|\x^o,\h)} d\hat\x^m\right]d\x^od\h\nonumber\\
&=&\sum_{i=1}^{d_m}\sum_{t\in\{0,1\}}\int_{\mH_t^i}p(\h,M_i\!=\!t)\int_{\x^o}p(\x^o|\h,m_i\!=\!t)\mbox{KL}\Big(p(\hat\x^m|\x^o,\h,m_i=t)||p(\hat\x^m|\x^o,\h)\Big)d\x^od\h,\nonumber
\eeqr
which achieves its minimum when
\beq
p(\hat\x^m|\x^o,\h,m_i=t)=p(\hat\x^m|\x^o,\h)\label{eqn7}
\eeq
for $t\in\{0,1\}$ and $i\in\{1,\cdots,d^m\}$.

Let $\B=(B_1,\cdots,B_{d^m})\in\{0,1\}^{d^m}$ be a random vector taking value $\b_i^0$ with probability $\frac{1}{d^m}$, where $\b_i^0=(1,\cdots,1,0,1,\cdots,1)$ is a $d_m$-dimensional vector with only the $i$th element being 0, $i=1,\cdots,d^m$. The hint vector $\H=\B\odot\M+0.5(1-\B)$. Note that $H_i=t$ means $M_i=t$ for $t\in\{0,1\}$ and $H_i=0.5$ implies nothing about $M_i$. With this $\H$, $D_i^*(\hat\x,\h)=t$ for $\h$ such that $\h_i=t$ and $t\in\{0,1\}$.

For any $\m=(m_1,\cdots,m_{d^m})\in\{0,1\}^{d_m}$ and $i\in\{1,\cdots,d^m\}$, let $\m_0$, $\m_1\in\{0,1\}^{d_m}$ be any two vectors such that they are the same as $\m$ on the $j$th element for $j\neq i$, and the $i$th components of $\m_0$ and $\m_1$  are 0 and 1, respectively. So $\m=\m_0$ if $m_i=0$ and $\m=\m_1$ if $m_i=1$. Define a realization of the hint vector $\H$ as $\h$ such that $h_j=m_j$ if $j\neq i$ and $h_j=0.5$ if $j=i$. Since $p(\h|m_i=t)>0$, by (\ref{eqn7}) we have
\beq p(\hat\x^m|\x^o,\h,m_i=0)=p(\hat\x^m|\x^o,\h,m_i=1).\label{eqn8}\eeq
Note
\beqr
p(\hat\x^m|\x^o,\h,m_i=t)=p(\hat\x^m|\x^o,\B=\b_i^0,\m=\m_t)=p(\hat\x^m|\x^o,m=\m_t),\label{eqn9}
\eeqr
where the first equation holds since $\{\h,m_i=t\}$ is equivalent to $\{\B=\b_i^0,\m=\m_t\}$, and the second equality holds due to the independence of $\B$ to the other variables. Combing (\ref{eqn8}) and (\ref{eqn9}), we have $p(\hat\x^m|\x^o,\m_0)=p(\hat\x^m|\x^o,\m_1)$.

Let $\textbf{1}=(1,\cdots,1)$ and $\m$ be any vector in $\{0,1\}^{d_m}$. There exists a sequence of vectors $\m_1',\cdots,\m_L'$ such that $\m_l'$ and $\m_{l+1}'$ only differs on one component and $\m_1'=\m$ and $\m_L'=\textbf{1}$. By the arguments above, we have
$$p(\hat\x^m|\x^o,\m)=p(\hat\x^m|\x^o,\m_1')=\cdots=p(\hat\x^m|\x^o,\m_L')=p(\hat\x^m|\x^o,\textbf{1}).$$
Note that $p(\hat\x^m|\x^o,\textbf{1})=p(\x^m|\x^o,\textbf{1})$. And by the MAR assumption we have $p(\x^m|\x^o,\textbf{1})=p(\x^m|\x^o)$. So
$$p(\hat\x^m|\x^o)=\sum_{\m\in\{0,1\}^{d_m}}p(\M=\m)p(\hat\x^m|\x^o,\m)=\sum_{\m\in\{0,1\}^{d_m}}p(\M=\m)p(\x^m|\x^o)=p(\x^m|\x^o),$$
which implies $p(\hat\X)=p(\X)$ as we need.

\vspace{0.4cm}
\subsection{Dataset Details and Hyper-parameters}
\label{A3}
The details of the five UCI datasets can be found in UCI machine learning repository \cite{UCI}  and the paper of GAIN \cite{Yoon18}.

The datasets Internet Loan and ADNI are available in the Supplementary Materials. And the number of variables at each data source and sample size of each response pattern are presented in below tables.

\begin{table}[h]
\caption{The number of variables at each data source and sample size of each response pattern for the Inernet Loan data.}
\label{table10}
\begin{center}
\begin{small}
\begin{tabular}{ccccccc}
		\toprule
		Response& \multicolumn{5}{c}{Data source}  & Sample  \\
\cline{2-6}
		Pattern   & \mbox{Card} (5) & Shopping (4) & Mobile (5) & Bureau (7) & Fraud (4) &  Size \\
		\hline
		1   & $\surd$   & $\surd$   & $\surd$   & $\surd$   & $\surd$  & 115 \\
		2   & $\surd$   & $\surd$   & $\surd$   & $\surd$   &   & 29  \\
		3   & $\surd$   & $\surd$   & $\surd$   &     &    & 220 \\
		4   & $\surd$   & $\surd$   &     & $\surd$   & $\surd$  & 232 \\
		5   & $\surd$   & $\surd$   &     & $\surd$   &   &  113 \\
		6   & $\surd$   & $\surd$   &     &     & &   222  \\
		7   & $\surd$   &     & $\surd$   &     &   &  11 \\
		8   & $\surd$   &     &     & $\surd$   & $\surd$  & 38  \\
		9   & $\surd$   &     &     & $\surd$   &   & 102  \\
		10  & $\surd$   &     &     &     &   &  302 \\
\hline
\multicolumn{6}{r}{Total} & 1384 \\
		\bottomrule
		\end{tabular}%
\end{small}
\end{center}
\vskip -0.25in
\end{table}

\begin{table}[h]
\caption{The number of variables at each data source and sample size of each response pattern for the ADNI data.}
\label{table11}
\begin{center}
\begin{small}
\begin{tabular}{cccccc}
		\toprule
		Response& \multicolumn{4}{c}{Data source}  & Sample  \\
\cline{2-5}
		Pattern   & \hspace{0.15cm} CSF (3) \hspace{0.15cm}& \hspace{0.15cm}PET (10) \hspace{0.15cm}& \hspace{0.15cm}MRI (10)\hspace{0.15cm} & \hspace{0.15cm}Gene (10)\hspace{0.15cm} & Size \\
		\hline
		1   & $\surd$   & $\surd$   & $\surd$   & $\surd$   & 413   \\
		2   & $\surd$   & $\surd$   & $\surd$   &     & 367  \\
		3   & $\surd$   & $\surd$   &    &    $\surd$   & 34   \\
		4   &    & $\surd$   &   $\surd$  & $\surd$   & 109  \\
		5   &    & $\surd$   &     & $\surd$     & 81  \\
		6   &   & $\surd$   & $\surd$     &       & 54   \\
		7   &    &     &    &     $\surd$    & 51 \\
		8   &   &     &  $\surd$    &     & 58  \\
\hline
\multicolumn{5}{r}{Total} & 1167\\
		\bottomrule
			\end{tabular}%
\end{small}
\end{center}
\end{table}

The original ADNI data is available at {\it http://adni.loni.usc.edu}. The number of variables from the last three sources in the original data are larger. We use feature screen methods (Fan and Lv, 2008) to screen out the most 10 important variables for each source for our experiment.

In all experiments, the depth of generator, discriminator and predictor in FragmGAN, GAIN and Auto-Encoder is set to be 3. The number of hidden nodes in each layer for generator and discriminator is $2d$, $d$ and $d$, respectively. The number of hidden nodes in each layer for predictor is $d$, $d/2$ and $1$, respectively. The activation function is {\it ReLu} except for the output layer that uses {\it sigmoid}. The training batch sizes $k_G$, $k_D$ and $k_P$ are all 64. The $\alpha$ in $\mathcal{L}_M$ is 10. For the cross-validation of FragmGAN, we search the value of $\gamma$ on the grid of $\{0.40, 0.41,\cdots, 0.59, 0.60\}$.

We use {\it Pytorch} to implement FragmGAN, GAIN, Auto-Encoder and MisGAN. The code of FragmGAN is available at the Supplementary Material. We use {\it Python} to implement:

MICE (package `fancyimpute', {\it https://github.com/iskandr/fancyimpute}),

MissForest (package `missingpy', {\it	https://github.com/softmechanics/missingpy}),

EM (package `impyute', {\it https://github.com/eltonlaw/impyute}),

and Matrix ({\it https://www.cnblogs.com/wuliytTaotao/p/10814770.html}).

We use $R$ to implement Model Averaging and FR-FI.

\vspace{0.4cm}

{\bf References}

Fan, J. and Lv, J. Sure independence screening for ultrahigh
dimensional feature space (with disccusions). {\it Journal of
Royal Statistical Society, Series B}, 70(5):849--911, 2008.


\end{document}